\documentclass[12pt]{article} 
\usepackage{aaai25}  
\usepackage{times}  
\usepackage{helvet}  
\usepackage{courier}  
\usepackage[hyphens]{url}  
\usepackage{graphicx} 
\usepackage{natbib}  
\usepackage{caption} 
\usepackage{color}

\usepackage{amsmath,amssymb,amsthm}
\newtheorem{theorem}{Theorem}
\newtheorem{lemma}[theorem]{Lemma}
\newtheorem{cor}[theorem]{Corollary}
\newtheorem{definition}{Definition}

\newtheorem*{theorem*}{Theorem}
\newtheorem*{lemma*}{Lemma}

\usepackage{tikz}
\usetikzlibrary{shapes, positioning}

\usepackage{color}
\definecolor{red}{rgb}{0.957,0.498,0.447}
\definecolor{lightred}{rgb}{0.965,0.702,0.675}
\definecolor{green}{rgb}{0.553,0.824,0.773}
\definecolor{lightgreen}{rgb}{0.729,0.890,0.863}
\definecolor{blue}{rgb}{0.498,0.698,0.835}
\definecolor{lightblue}{rgb}{0.702,0.820,0.906}
\definecolor{orange}{rgb}{0.969,0.718,0.427}
\definecolor{lightorange}{rgb}{0.988,0.831,0.631}
\definecolor{purple}{rgb}{0.749,0.737,0.855}

\usepackage{makecell}
\usepackage{multirow}
\usepackage{subcaption}
\usepackage{xfrac}

\title{PowerMLP: An Efficient Version of KAN}

\author {
    Ruichen Qiu\textsuperscript{\rm 1,\rm 2},
    Yibo Miao\textsuperscript{\rm 2,3},
    Shiwen Wang\textsuperscript{\rm 3},
    Lijia Yu\textsuperscript{\rm 4},
    Yifan Zhu\textsuperscript{\rm 2,3},
    Xiao-Shan Gao\textsuperscript{\rm 2}\thanks{Corresponding author.}
}
\affiliations {
    \textsuperscript{\rm 1}School of Advanced Interdisciplinary Sciences, UCAS, Beijing 100049, China\\
    \textsuperscript{\rm 2}Academy of Mathematics and Systems Science, CAS, Beijing 100190, China\\
    \textsuperscript{\rm 3}University of Chinese Academy of Sciences, Beijing 101408, China\\
    \textsuperscript{\rm 4}Institute of Software, CAS, Beijing 100190, China\\
    qiuruichen20@mails.ucas.ac.cn, miaoyibo@amss.ac.cn, wangshiwen21@mails.ucas.ac.cn\\ yulijia@ios.ac.cn, zhuyifan@amss.ac.cn, xgao@mmrc.iss.ac.cn
}

\begin{document}

\maketitle

\begin{abstract}
The Kolmogorov-Arnold Network (KAN) is a new network architecture known for its high accuracy in several tasks such as function fitting and PDE solving.
The superior expressive capability of KAN arises from the Kolmogorov-Arnold representation theorem and learnable spline functions.
However, the computation of spline functions involves multiple iterations, which renders KAN significantly slower than MLP, thereby increasing the cost associated with model training and deployment.
The authors of KAN have also noted that {\em ``the biggest bottleneck of KANs lies in its slow training. KANs are usually 10x slower than MLPs, given the same number of parameters.''}
To address this issue, we propose a novel MLP-type neural network PowerMLP that employs simpler non-iterative spline function representation, offering approximately the same training time as MLP while theoretically demonstrating stronger expressive power than KAN.
Furthermore, we compare the FLOPs of KAN and PowerMLP, quantifying the faster computation speed of PowerMLP. Our comprehensive experiments demonstrate that PowerMLP generally achieves higher accuracy and a training speed about 40 times faster than KAN in various tasks.
\end{abstract}

%

\section{Introduction}

A long-standing problem of deep learning is the identification of more effective neural network architectures.
The Kolmogorov-Arnold Network (KAN), introduced by \citet{liu2024kan}, presents a new architecture.
Unlike traditional MLP, which places activation functions on nodes, KAN employs learnable univariate spline functions as activation functions placed on edges.
The expressive power of KAN is derived from the Kolmogorov-Arnold representation theorem and the property of spline functions.
Due to its exceptional expressiveness, KAN achieves high accuracy and interpretability in multiple tasks such as function fitting and PDE solving.
Significant performance improvements using KAN have been observed in time series prediction~\cite{ inzirillo2024sigkansignatureweightedkolmogorovarnoldnetworks}, graph data processing~\cite{kiamari2024gkangraphkolmogorovarnoldnetworks}, and explainable natural language processing~\cite{boris2024kolmogorov}.

Unfortunately, despite the impressive performance of KAN, it faces a critical drawback: slow inference and training speeds, which increase the cost associated with training and deploying the model.
\citet{liu2024kan} highlight in the KAN paper that {\em ``the biggest bottleneck of KANs lies in its slow training. KANs are usually 10x slower than MLPs, given the same number of parameters.''}
The inefficient training and inference latency of KAN plays a crucial role in ensuring a positive user experience and low computational resource requirements.

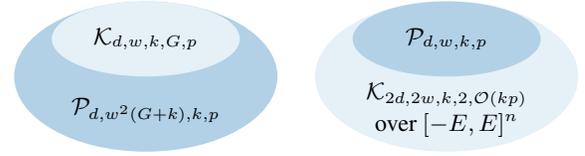
\begin{figure}[t]
    \centering
    \begin{tikzpicture}
        \node (powermlp) at (-3.5,-0.5) [ellipse, draw=none, fill=white!40!blue, minimum width=3.5cm, minimum height=2cm, inner sep=0pt] {};
        \node (kan) at (-3.5,0) [ellipse, draw=none, fill=white!80!blue, minimum width=2.5cm, minimum height=1cm, inner sep=0pt] {};
        \node (text1) at (-3.5,0) [draw=none,fill=none] {\small \textcolor{black}{$\mathcal{K}_{d,w,k,G,p}$}};
        \node (text2) at (-3.5,-0.95) [draw=none,fill=none]{\small \textcolor{black}{$\mathcal{P}_{d,w^2(G+k),k,p}$}};
        \node (kan2) at (0.5,-0.5) [ellipse, draw=none, fill=white!80!blue, minimum width=3.5cm, minimum height=2cm, inner sep=0pt] {};
        \node (powermlp2) at (0.5,0) [ellipse, draw=none, fill=white!40!blue, minimum width=2.5cm, minimum height=1cm, inner sep=0pt] {};
        \node (text1) at (0.5,0) [draw=none,fill=none] {\small \textcolor{black}{$\mathcal{P}_{d,w,k,p}$}};
        \node (text2) at (0.5,-0.75) [draw=none,fill=none]{\small \textcolor{black}{$\mathcal{K}_{2d,2w,k,2,\mathcal{O}(kp)}$}};
        \node (text3) at (0.5,-1.15) [draw=none,fill=none]{\small \textcolor{black}{over $[-E,E]^n$}};
\end{tikzpicture}
    \captionof{figure}{
    PowerMLPs define a strictly larger function space than KANs over $\mathbb{R}^n$ (Corollary \ref{cor-k2p1}), and define the same function space over $[-E,E]^n$ for any $E\in\mathbb{R}_+$ (Corollary \ref{cor-p2k}), where $n$ is the input dimension.
    $\mathcal{P}_{d,w,k,p}$ is the set of all PowerMLP networks with depth $d$, width $w$, $k$-th power ReLU activation function, and $p$ nonzero parameters.
${\mathcal{K}}_{d,w,k,G,p}$ is the set of all KAN networks with depth $d$, width $w$, using $(k,G)$-spline (see Eq. \eqref{eq-sfun}), and $p$ nonzero parameters.
}
  \label{fig:funcspace}
\end{figure}

Upon examining the structure of KAN,
we contribute the primary cause of the slow computation speed to the spline activation function.
Specifically, a $k$-order spline function is a linear combination of $k$-order B-splines, and each $k$-order B-spline requires construction through $\mathcal{O}(k^2)$ iterations by de Boor-Cox formula in KAN's paper, resulting in slow computations.
To address this issue, we need to identify a more efficient method to obtain B-splines, moving away from the recursive computation of the de Boor-Cox formula.

The $k$-order B-spline was initially defined as
$k$-th divided difference of the truncated power function~\cite{curry1947on}. Consequently, B-splines can also be expressed as linear combinations of powers of ReLUs~\cite{greville1969theory}.
Inspired by this, we introduce a novel MLP-type nueral network by incorporating a basis function into the MLP with powers of ReLU as activation function, termed \textbf{PowerMLP}, as shown in Figure \ref{fig:rrpn}.
%
%
From Figure \ref{fig:funcspace}, PowerMLPs define a strictly larger function space than KANs over $\mathbb{R}^n$ and define the same function space over $[-E,E]^n$, indicating that PowerMLP can serve as a viable substitute for KAN.
Intuitively, the B-spline can be represented by PowerMLPs without iterative recursion, leading to faster computation. In addition, the activation functions of PowerMLP are on the nodes instead of the edges, inhering the fast training advantages of MLP.
We further choose Floating Point Operations (FLOPs) as the metric and demonstrate that the FLOPs of KAN are more than 10 times those of PowerMLP, theoretically explaining why PowerMLP is significantly faster than KAN.

We conducted extensive experiments to validate the advantages of PowerMLP in various experimental settings, including those of KAN and more complicated tasks such as image classification and language processing.
The results show that PowerMLP trains about 40 times faster than KAN and achieves the best performance in most experiments.

\section{Related Work}

\noindent{\bf Network Architecture.}
Multilayer Perceptron (MLP) is a basic form of neural networks used primarily for supervised learning tasks~\cite{haykin1998neural}. One of the activation functions commonly used with MLP is the rectified linear unit (ReLU)~\cite{glorot2011deep}, which is further extended to LeakyReLU~\cite{maas2013rectifier}, PReLU~\cite{he2015delving}, RePU~\cite{li2020better,li2020powernet}, and GeLU~\cite{hendrycks2023gaussianerrorlinearunits}.
\citet{lecun1989backpropagation} introduced CNNs specialized for images and \citet{vaswani2017attention} introduced Transformers, which have become a cornerstone for large language models.
%
PowerMLP can be obtained from RePU~\cite{li2020better,li2020powernet} by adding a basis function to each layer.

\noindent
{\bf Kolmogorov-Arnold Network.}
KAN~\cite{liu2024kan} is a new network architecture using learnable activation functions on edges. In small-scale AI + Science tasks, KAN outperforms MLP in terms of both accuracy and interpretability.
Furthermore, KAN also achieves remarkable performance in time series prediction~\cite{xu2024kolmogorovarnoldnetworkstimeseries, inzirillo2024sigkansignatureweightedkolmogorovarnoldnetworks}, graph-structured data processing~\cite{kiamari2024gkangraphkolmogorovarnoldnetworks, decarlo2024kolmogorovarnoldgraphneuralnetworks}, and explainable natural language processing~\cite{boris2024kolmogorov}.
Meanwhile, notable approaches are taken to improve KAN's interpretability and performance, including incorporating wavelet functions into KAN~\cite{bozorgasl2024wavkanwaveletkolmogorovarnoldnetworks}, replacing the basis function to rational functions~\cite{aghaei2024rkanrationalkolmogorovarnoldnetworks}, and faster implementation by approximation of radial basis functions~\cite{li2024kolmogorovarnoldnetworksradialbasis}.
However, the computational speed of KAN and these variants is significantly slower than MLP, since they have a structure similar to that of KAN.

\noindent{\bf PowerMLP vs. KAN.}
PowerMLP can be considered as an MLP-type representation of KAN.
In Section 6 of \cite{liu2024kan}, KANs' ``limitations and future directions'' are discussed.
Most of these limitations of KAN can be eliminated or mitigated by PowerMLP.
(1) In algorithmic aspects, the authors noted that ``KANs are usually 10x slower than MLPs.'' PowerMLPs, which are about 40 times faster than KANs in our experiments, eliminate this limitation.
(2) In mathematical aspects, the authors call for a more general approximation result beyond the depth-2 Kolmogorov-Arnold representations.
Such a result is given in Corollary \ref{cor-app} using our proposed PowerMLP.
(3) In the application aspects, the authors call for  ``integrating KANs into current architectures.''
This can be done by replacing the MLP in these architectures with PowerMLP more naturally than KAN.
%

\section{Preliminaries}
In this section, we introduce the preliminary knowledge of Kolmogorov-Arnold Network (KAN)~\cite{liu2024kan}.

\subsection{Spline Function}
The following {\em de Boor-Cox formula}~\cite{cox1972numerical,deBoor1978practical} for B-spline is used to define KAN.
\begin{definition}[B-spline]\label{def:bspline}
    Let $t:=(t_j)$ be a nondecreasing sequence of real numbers, called {\em knot sequence}.
    The zero-order B-spline on $(t_j,t_{j+1})$ is defined as~\footnote{Note that de Boor-Cox formula defines the function in Eq. \eqref{eq-bsp} as order $1$. To keep the same mark with KAN, we adjust it to $0$.}
    {\small
    \begin{equation}
    \label{eq-bsp}
        B_{j,0,t}(x)=\left\{\begin{aligned}
            &1,\quad&t_j\leq x<t_{j+1},\\
            &0,&\mathrm{otherwise}.
        \end{aligned}\right.
    \end{equation}}
    \hskip-4pt
    Then the $j$-th $k$-order {\em normalized B-spline} for the {\em knot sequence} $t$ is defined recursively as
    {\small
    \begin{equation}\label{eq:recur}
        \begin{aligned}
            B_{j,k,t}(x)
            &=\frac{x-t_j}{t_{j+k}-t_j}B_{j,k-1,t}(x)\\
            &\quad+\frac{t_{j+k+1}-x}{t_{j+k+1}-t_{j+1}}B_{j+1,k-1,t}(x),\quad\mathrm{for\,}k\geq1.
        \end{aligned}
    \end{equation}}
\end{definition}

To maintain consistency with KAN's setting, let
$t=(\underline{t_{-k}, \cdots,t_{-1}},$
$t_0,t_1, \cdots,t_G,\underline{t_{G+1}, \cdots,t_{G+k}})$
be an \textit{increasing} knot sequence, called {\em a $(k,G)$-grid}.
Through a linear combination of B-splines on $t$, we provide the definition of spline functions as follows.
\begin{definition}[Spline Function]
    \label{def-splinef}
    Let $t$ be a $(k,G)$-grid. A {\em $k$-order spline function} for the knot sequence $t$ is given by the following linear combination of B-splines
    {\small \begin{align}
    \label{eq-sfun}
        \mathrm{spline}_{k,G}(x)=\sum_{j=-k}^{G-1} c_jB_{j,k,t}(x),
    \end{align}}
    \hskip-4pt where $c_j\in\mathbb{R}$ are the coefficients of the spline function. For simplicity, \eqref{eq-sfun} is called a $(k,G)$-spline function.
\end{definition}

\subsection{Kolmogorov-Arnold Networks}\label{sec:kan}
 A  KAN network \cite{liu2024kan} is a composition of $L$ layers: given an input vector $\mathbf{x}\in\mathbb{R}^{n_0}$, the output of KAN is
{\small \begin{equation*}
    \mathrm{KAN}(\mathbf{x})=(\Phi_{L-1}\circ \cdots\circ\Phi_1\circ\Phi_0)(\mathbf{x}),
\end{equation*}}
\hskip-4pt
where $\Phi_\ell$ is the function matrix corresponding to the {\em $\ell$-th layer}. The dimension of the input vector of $\Phi_\ell$ is denoted as $n_\ell$, and $\Phi_\ell$ is defined below:
\begin{equation*}\label{eq:kanlayer}
    \Phi_\ell(\cdot)=
    \begin{pmatrix}
        \phi_{l,1,1}(\cdot) & \phi_{l,1,2}(\cdot) & \cdots & \phi_{l,1,n_l}(\cdot)\\
        \phi_{l,2,1}(\cdot) & \phi_{l,2,2}(\cdot) & \cdots & \phi_{l,2,n_l}(\cdot)\\
        \vdots & \vdots & & \vdots\\
        \phi_{l,n_{l+1},1}(\cdot) & \phi_{l,n_{l+1},2}(\cdot) & \cdots & \phi_{l,n_{l+1},n_l}(\cdot)\\
    \end{pmatrix},
\end{equation*}
where $\phi_{\ell,q,p}$ is a residual activation function:
{
\begin{equation}\label{eq:kan_activation}
    \phi_{\ell,q,p}(x_p)=u_{\ell,q,p}b(x_p)+v_{\ell,q,p}\mathrm{spline}(x_p).
\end{equation}}
$b(x)$ is a non-parameter basis function~\footnote{Basis function is $b(x)={x}/{(1+e^{-x})}$ in KAN's paper.}
similar to residual shortcut, which is included to improve training. $\mathrm{spline}(x)$ is defined in Definition \ref{def-splinef}.
The KAN defined above, called a KAN of $k$-order, has depth $L$, width $W=\max_{i=0}^{L-1} n_i$, and $O(W^2L(G+k))$ parameters.

\section{PowerMLP}
\label{sec-theory}

While using spline functions as activations enables KAN to achieve excellent performance, computing spline functions involves multiple iterations, leading to a high number of FLOPs (see Table \ref{tab:flops}). This results in KAN's computation speed being slower than that of MLP, thereby increasing the cost associated with training and deploying the model.

To replace KAN with a network structure that offers similar expressiveness but faster computation, we
present a simpler representation of spline functions,
avoiding the recursive calculations of de Boor-Cox formula. In Section \ref{sec:powermlp}, we introduce a novel network, named \textbf{PowerMLP}.
Through theorems in Sections \ref{sec:kan2powermlp} and \ref{sec:powermlp2kan}, we demonstrate that KAN and PowerMLP define the same function space over bounded intervals, indicating their interchangeability.
Additionally, in Section \ref{sec-FLOP}, we compare the FLOPs of KAN and PowerMLP, revealing that PowerMLP theoretically achieves faster speeds than KAN.
Proofs are given in Appendix A.

\subsection{PowerMLP}
\label{sec:powermlp}
Referring back to the de Boor-Cox formula in Eq. \eqref{eq:recur}, we attribute the slow computation speed of KAN to the $\mathcal{O}(k^2)$ iterations of calculation required for constructing $k$-order B-spline.
According to the initial definition of $k$-order B-spline from the $k$-th divided difference of the truncated power function~\cite{curry1947on} and subsequent work by \citet{greville1969theory}, we express B-splines through a non-iterative approach as linear combinations of powers of ReLU function.
Thus, we introduce a novel MLP network that incorporates a basis function into the MLP with powers of ReLU as activations, termed \textbf{PowerMLP}.

The $k$-th power of ReLU~\cite{Mhaskar1993ApproximationPO, li2020better} is defined as:
\begin{equation}
\label{eq-sigmak}
    \sigma_k(x)=\left(\mathrm{ReLU}(x)\right)^k=\left(\max(0,x)\right)^k
    \quad k\in \mathbb{Z}_+.
\end{equation}
The fully connected feedforward neural network with the $k$-th power of ReLU as the activation function is referred to as {\em ReLU-$k$ MLP}. We integrate a basis function into ReLU-$k$ MLP and define PowerMLP following the form of KAN.
\begin{definition}[PowerMLP]
    A PowerMLP is a neural network composition of $L$ layers:
    {\small\begin{equation}\label{eq:rrndef}
        \begin{aligned}
            &\mathrm{PowerMLP}(\mathbf{x})=(\Psi_{L-1}\circ\cdots\circ\Psi_1\circ\Psi_0)(\mathbf{x}), \quad\mathrm{where}\\
            &\Psi_\ell(\mathbf{x}_\ell)=\left\{\begin{aligned}
                &
                \alpha_\ell b(\mathbf{x}_\ell)+ \sigma_k(\omega_\ell\mathbf{x}_\ell+\gamma_\ell),&\mathrm{for\,} \ell<L-1,\\
                &\omega_{L-1}\mathbf{x}_{L-1}+\gamma_{L-1},&\mathrm{for\,} \ell=L-1.
            \end{aligned}\right.
        \end{aligned}
    \end{equation}}
    \hskip-4pt
$\alpha_\ell\in\mathbb{R}^{m_{\ell+1}\times m_{\ell}},
\omega_\ell\in\mathbb{R}^{m_{\ell+1}\times m_{\ell}},
\gamma_\ell\in\mathbb{R}^{m_{\ell}\times 1}$ are trainable parameters. $b(\mathbf{x})$
is a basis function that performs the same operation as the basis function in KAN on each component of $\mathbf{x}$,
and $\sigma_k$ is the ReLU-$k$ activation function.
A PowerMLP using $\sigma_k$ as activation function
is called $k$-order PowerMLP.
The {\em width and depth of the PowerMLP} are defined to be $\mathrm{width}=\max_{l=0}^{L-1}\{m_l\},\mathrm{depth}=L$.
Refer to Figure \ref{fig:rrpn} for an illustration.
\end{definition}

Through non-parameter activation functions $\sigma_k$ and non-iteration calculation, PowerMLP computes faster than KAN.
Furthermore, they can represent each other within bounded intervals,
a condition usually met in practical scenarios.

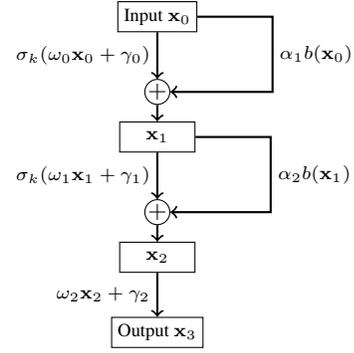
\begin{figure}[t]
    \centering
    \scriptsize
    \begin{tikzpicture}[node distance=1cm, every node/.style={fill=white, draw=black, rectangle, minimum width=1cm, minimum height=0.4cm}]

        \node (layer1) {Input $\mathbf{x}_0$};
        \node (layer2) [below of=layer1,node distance=1.6cm] {$\mathbf{x}_1$};
        \node (layer3) [below of=layer2,node distance=1.6cm] {$\mathbf{x}_2$};
        \node (layer4) [below of=layer3] {Output $\mathbf{x}_3$};
        \node (plus1) [circle, draw, fill=white, minimum size=0.1cm, inner sep=0pt, below of=layer1, node distance=1cm] {\small $+$};
        \node (plus2) [circle, draw, fill=white, minimum size=0.1cm, inner sep=0pt, below of=layer2, node distance=1cm] {\small $+$};

        \draw [thick,->] (layer1) -- node[left,draw=none,fill=none] {$\sigma_k(\omega_0\mathbf{x}_0+\gamma_0)$} (plus1) ;
        \draw [thick,->] (plus1) -- (layer2);
        \draw [thick,->] (layer2) -- node[left,draw=none,fill=none] {$\sigma_k(\omega_1\mathbf{x}_1+\gamma_1)$} (plus2);
        \draw [thick,->] (plus2) -- (layer3);
        \draw [thick,->] (layer3) -- node[left,draw=none,fill=none] {$\omega_2\mathbf{x}_2+\gamma_2$} (layer4);
        \draw [thick,->] (layer1.east) -- ++(1cm, 0) |- node[right,yshift=0.5cm,draw=none,fill=none] {$\alpha_1 b(\mathbf{x}_0)$} (plus1);
        \draw [thick,->] (layer2.east) -- ++(1cm, 0) |- node[right,yshift=0.5cm,draw=none,fill=none] {$\alpha_2 b(\mathbf{x}_1)$} (plus2);

    \end{tikzpicture}
    \captionof{figure}{Structure of a 3-layer PowerMLP. The first two layers are calculated by: (1) affine transformation, (2) $k$-th power of ReLU activation, (3) addition with a basis function. The last layer contains only an affine transformation.}
    \label{fig:rrpn}
\end{figure}

\subsection{PowerMLP can represent KAN}
\label{sec:kan2powermlp}
In this section, we show that any KAN can be represented by a PowerMLP.
We first give the following connection between B-splines and the k-th power of ReLU $\sigma_{k}$.

\begin{lemma}[Represent the B-spline with powers of ReLU]\label{lem:bspline}
If $t_u \neq t_v (\forall u \neq v)$, then the $k$-order B-spline on the knot sequence $t=(t_j, \cdots, t_{j+k+1})$ can be represented as a linear combination of $\sigma_{k}$ functions:
    {\small\begin{align*}
    \label{eq-Bs}
        B_{j,k,t}(x)&=\sum_{i=j}^{j+k+1}\frac{t_{j+k+1}-t_j}{\prod_{l=j}^{j+k+1,l\neq i}(t_l-t_i)}\sigma_{k}(x-t_i).
    \end{align*}}
\end{lemma}
Since each spline activation function in KAN is a linear combination of B-splines, we can represent it as a linear combination of ReLU-$k$. By performing this operation on each component of the input vector $\mathbf{x}$ and incorporating a basis function, we represent a KAN layer as a PowerMLP layer. Since KAN and PowerMLP are both compositions of layers, it means that a KAN can be represented as a PowerMLP. We have the following theorem.
\begin{theorem}[KAN is a Subset of PowerMLP]
\label{th-k4p}
    Fix the input dimension of networks as $n$. Let $\mathcal{K}_{d,w,k,G,p}$ be the set of all KAN networks with depth $d$, width $w$, $p$ nonzero parameters, using $(k,G)$-spline; and $\mathcal{P}_{d,w,k,p}$ be the set of all PowerMLPs with depth $d$, width $w$, $k$-th power of ReLU and $p$ nonzero parameters. Then it holds
    \begin{equation}
    \label{eq-k2p0}
    \mathcal{K}_{d,w,k,G,p}\subset\mathcal{P}_{d,w^2(G+k),k,p}.
    \end{equation}
\end{theorem}
By Theorem \ref{th-k4p} and the fact that
a spline function is zero outside a certain interval (see Eq. \eqref{eq-bsp}) while PowerMLPs include all polynomials \cite{li2020better}, we have
\begin{cor}
\label{cor-k2p1}
 PowerMLPs define a strictly larger function space than KANs over $\mathbb{R}^n$.
\end{cor}

\subsection{KAN can represent PowerMLP over Intervals}
\label{sec:powermlp2kan}

In this section, we prove the inclusion relationship in another direction.
A PowerMLP layer
\begin{equation}
\label{eq-player}
\mathbf{z}=\sigma_{k}(\omega\mathbf{x}+\gamma)+\alpha b(\mathbf{x})
\end{equation}
can be decomposed into $3$ operations: (1) an affine transformation: $\mathbf{y}=\omega\mathbf{x}+\gamma$;
(2) a ReLU-$k$ activation: $\mathbf{u}=\sigma_{k}(\mathbf{y})$; (3) an addition with basis function: $\mathbf{z}=\mathbf{u}+\alpha b(\mathbf{x})$.
By Lemmas \ref{lem:affine} and \ref{lem:repu}, operations (1) and (2) can be represented by spline functions,
while operation (3) can easily achieve.

\begin{lemma}[Affine Transformation]\label{lem:affine}
Consider an affine transformation on $\mathbb{R}$:
 $       \mathcal{A}(x)=\omega x+\gamma$.
For any $G$, we can find a $(k,G)$-grid
$t=(\underline{t_{-k},\cdots,t_{-1}},t_0,t_1,\cdots,t_G,$
$\underline{t_{G+1},\cdots,t_{G+k}}),$
    and a $k$-order spline function
\begin{equation*}
    \begin{aligned}
        &\mathrm{spline}_{k,G}(x)=\sum_{j=-k}^{G-1}c_jB_{j,k,t}(x),\\
    \end{aligned}
    \end{equation*}
where $c_j=\left(\sfrac{\sum_{i=j+1}^{j+k}t_i}{k}\right)\omega+\gamma,\,k>0$,
     such that $\mathcal{A}(x)=\mathrm{spline}_{k,G}(x)$ for $t_0\le x\le t_G$.
\end{lemma}

\begin{lemma}[ReLU-$k$ Function]\label{lem:repu}
We can find a $(k,2)$-grid $t=(\underline{t_{-k},\cdots,t_{-1}},t_0,0,t_2,\underline{t_3,\cdots,t_{k+2}})$ and a $k$-order spline function defined on $t$
    \begin{equation*}
        \mathrm{spline}_{k,2}(x)=\sum_{j=-k}^{1}\left[\left(\prod_{l=j+1}^{j+k}\sigma_1(t_l)\right) B_{j,k,t}(x)\right],
    \end{equation*}
such that $\sigma_{k}(x)=\mathrm{spline}_{k,2}(x)$ for $t_0\le x\le  t_2$.
\end{lemma}

\begin{figure}[t]
    \centering
    \small
    \begin{tikzpicture}
        \node (layer1-1) at (-0.2,0.7) [circle, draw, fill=white, minimum size=0.8cm, inner sep=0pt] {$x_1$};
        \node (layer1-2) at (-0.2,-0.1) [circle, draw, fill=white, minimum size=0.8cm, inner sep=0pt] {$x_2$};
        \node (layer1-3) at (-0.2,-0.7) [] {$\vdots$};
        \node (layer1-4) at (-0.2,-1.6) [circle, draw, fill=white, minimum size=0.8cm, inner sep=0pt] {$x_n$};

        \node (layer2-1) at (3.5,1.6) [circle, draw, fill=white, minimum size=0.8cm, inner sep=0pt] {$y_1$};
        \node (layer2-2) at (3.5,1.0) [] {$\vdots$};
        \node (layer2-3) at (3.5,0.1) [circle, draw, fill=white, minimum size=0.8cm, inner sep=0pt] {\scriptsize $y_m$};
        \node (layer2-4) at (3.5,-0.9) [circle, draw, fill=white, minimum size=0.8cm, inner sep=0pt] {\scriptsize $y_{m+1}$};
        \node (layer2-5) at (3.5,-1.7) [circle, draw, fill=white, minimum size=0.8cm, inner sep=0pt] {\scriptsize $y_{m+2}$};
        \node (layer2-6) at (3.5,-2.3) [] {$\vdots$};
        \node (layer2-7) at (3.5,-3.1) [circle, draw, fill=white, minimum size=0.8cm, inner sep=0pt] {\scriptsize $y_{m+n}$};

        \node (layer3-1) at (7.2,1) [circle, draw, fill=white, minimum size=1cm, inner sep=0pt] {$z_1$};
        \node (layer3-2) at (7.2,-0.4) [] {$\vdots$};
        \node (layer3-3) at (7.2,-2) [circle, draw, fill=white, minimum size=1cm, inner sep=0pt] {$z_m$};

        \draw [thick,color=blue,->] (layer1-1) -- node[above, midway, sloped, yshift=0.3cm] {\scriptsize $\phi_{1,q,p}(x_p)=\omega_{q,p}x_p+\gamma_{q,p}$} (layer2-1);
        \draw [thick,color=blue,->] (layer1-2) -- node[above, yshift=-0.4cm, xshift=-0.3cm, draw=none,fill=none] {} (layer2-1);
        \draw [thick,color=blue,->] (layer1-4) -- node[above, yshift=-0.5cm, xshift=0.5cm, draw=none,fill=none] {} (layer2-1);
        \draw [thick,color=blue,->] (layer1-1) -- node[above, xshift=0.5cm, draw=none,fill=none] {} (layer2-3);
        \draw [thick,color=blue,->] (layer1-2) -- node[above, yshift=-0.5cm, xshift=-0.3cm, draw=none,fill=none] {} (layer2-3);
        \draw [thick,color=blue,->] (layer1-4) -- node[above, yshift=-0.6cm, xshift=0cm, draw=none,fill=none] {} (layer2-3);
        \draw [thick,color=red,->] (layer1-1) -- node[above, xshift=0cm, yshift=0cm, draw=none,fill=none] {} (layer2-4);
        \draw [thick,color=red,->] (layer1-2) -- node[above, xshift=0cm, yshift=0cm, draw=none,fill=none] {} (layer2-5);
        \draw [thick,color=red,->] (layer1-4) -- node[below, xshift=0cm, yshift=-0cm, midway, sloped] {\scriptsize \makecell{$\phi_{1,q,p}(x_p)=\delta_{q-m,p}x_p$}} (layer2-7);

        \draw [thick,color=blue!70!black,->] (layer2-1) -- node[above, draw=none,fill=none,midway,sloped,yshift=0.1cm] {\scriptsize \makecell{$\phi_{2,r,q}(y_q)=\delta_{r,q}\sigma_k(y_q)$}} (layer3-1);
        \draw [thick,color=blue!70!black,->] (layer2-3) -- node[above, draw=none,fill=none] {} (layer3-3);
        \draw [thick,color=red!70!black,->] (layer2-4) -- node[above, draw=none,fill=none] {} (layer3-1);
        \draw [thick,color=red!70!black,->] (layer2-4) -- node[above, draw=none,fill=none] {} (layer3-3);
        \draw [thick,color=red!70!black,->] (layer2-5) -- node[above, draw=none,fill=none] {} (layer3-1);
        \draw [thick,color=red!70!black,->] (layer2-5) -- node[above, draw=none,fill=none] {} (layer3-3);
        \draw [thick,color=red!70!black,->] (layer2-7) -- node[above, draw=none,fill=none] {} (layer3-1);
        \draw [thick,color=red!70!black,->] (layer2-7) -- node[below, draw=none,fill=none,midway,sloped,yshift=-0.1cm] {\scriptsize \makecell{$\phi_{2,r,q}(y_q)=\alpha_{r,q-m}b(y_q)$}} (layer3-3);
    \end{tikzpicture}
    \captionof{figure}{Represent a PowerMLP layer with a 2-layer KAN. $\delta_{ij}$ equals to $1$ if $i=j$ and $0$ otherwise.
    The first layer represents the affine transformation $y_q=\sum_{p=1}^n\omega_{q,p}x_p+\gamma_{q,p}$ for $1\leq q\leq m$ and keeps $y_q=x_{q-m}$ for $m+1\leq q\leq m+n$. The second layer represents the ReLU-$k$ activation and adds the basis function: $z_r=\sigma_k(y_r)+\sum_{q=m+1}^{m+n}\alpha_{r,q-m}b(y_q)$.}
    \label{fig:2layerkan}
\end{figure}
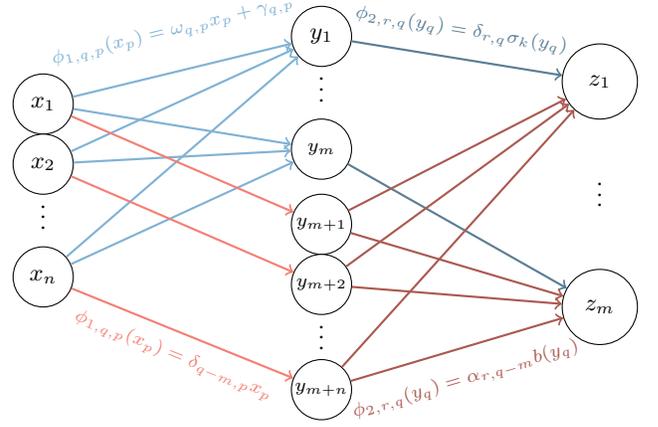

By the above lemmas, we can use two layers of spline functions to represent operations (1) and (2)  successively, as shown in Figure \ref{fig:2layerkan}. For the first layer, we set the coefficients before the basis function to be $0$. By Lemma \ref{lem:affine}, we can represent affine transformation $\mathcal{A}(x)=\omega x+\gamma$ by activation functions $\phi_{1,q,p}(x_p)$ (see Eq. \eqref{eq:kan_activation}).
Then, for $1\leq p \leq m$, we take the activation function as follows:
\begin{align*}
    \phi_{1,q,p}(x_p)=
    \left\{\begin{aligned}
        &\omega_{q,p}x_p+\gamma_{q,p},&\mathrm{for\,}1\leq q\leq m,\\
        &\delta_{q-m,p}x_p,&\mathrm{for\,}1\leq q-m\leq n,\\
    \end{aligned}\right.
\end{align*}
where $\delta_{ij}$ equals $1$ if $i=j$ and $0$ otherwise. Thus, we have $y_q=\sum_{p=1}^n\omega_{q,p}x_p+\gamma_{q,p}$ for $1\leq q\leq m$ and $y_q=x_{q-m}$ for $m+1\leq q\leq m+n$.
In the second layer, we represent the addition of ReLU-$k$ activation and the basis function. By Lemma \ref{lem:repu}, spline function can represent  $\sigma_k$. So for $1\leq r\leq m$, we set $\phi_{2,r,q}(y_q)$ as follows:
\begin{align*}
    \phi_{2,r,q}(y_q)=
\left\{\begin{aligned}
        &\delta_{r,q}\sigma_k(y_q),&\mathrm{for\,}1\leq q\leq m,\\
        &\alpha_{r,q-m}b(y_q),&\mathrm{for\,}1\leq q-m\leq n.\\
    \end{aligned}\right.
\end{align*}

By direct computation, we show that the output of the two-layer KAN is the same as our PowerMLP layer in Eq. \eqref{eq-player}.
Hence, we derive Theorem \ref{th-p2k} and Corollary \ref{cor-p2k}. Then by \cite[Theorem 3.3]{li2020better}, we obtain a general approximation result for KAN in Corollary \ref{cor-app}.

\begin{theorem}[PowerMLP is a subset of KAN over interval]
\label{th-p2k}
Use notations in Theorem \ref{th-k4p}.
    For any $E\in\mathbb{R}_+$, it holds
    \begin{equation}
    \label{eq-P2K}
    \mathcal{P}_{d,w,k,p}\subset 
    {\mathcal{K}}_{2d,2w,k,2,\mathcal{O}(kp)}
    \mathrm{\,over\,} [-E,E]^n.
    \end{equation}
\end{theorem}

\begin{cor}
\label{cor-p2k}
PowerMLPs and KANs define the same function space over $[-E,E]^n$ for any $E\in\mathbb{R}_+$.
\end{cor}
\begin{cor}
\label{cor-app}
Let $f$ be a continuous, first-order differentiable function on $[-1,1]^n$, which satisfies $\sum_{i=1}^n\int_{[-1,1]^n}(\partial_{x_i}f(x))^2dx\leq 1$. Then, for any $\epsilon\in(0,1)$, a $2$-order KAN $K$ requires at most $\mathcal{O}(n\log_2\frac{1}{\epsilon})$ layers and $\mathcal{O}(\epsilon^{-n})$ nonzero parameters to ensure
$\lVert K-f\rVert_{L_2}\leq\epsilon.$
\end{cor}

\subsection{FLOPs: Comparison of Computing Cost}
\label{sec-FLOP}
In this section, we show that PowerMLP exhibits significantly faster training and inference speeds compared to KAN in terms of FLOPs metric.
FLOPs, an acronym for \textit{Floating Point Operations}, is a metric used to quantify the computational complexity of neural networks, especially in frameworks like PyTorch~\cite{molchanov2017pruning}. It measures the number of floating-point operations required to perform one forward pass through the network.

Following \citet{yu2024kanmlpfairercomparison},
we consider FLOPs for any arithmetic operations like $+,-,\times,\div$ to be $1$,
and for Boolean operations to be $0$.
Meanwhile, any operation of comparing two numbers is set to be $0$ FLOPs, which means that the FLOPs of ReLU function are $0$. We denote FLOPs of basis function as $\lambda$. Then we can calculate FLOPs of one layer of MLP (with ReLU), KAN and PowerMLP below:
{\small
\begin{align*}
    &\mathcal{F}_\mathrm{MLP}=2 d_{in} d_{out},\\
    &\mathcal{F}_\mathrm{KAN}=d_{in}d_{out}(9kG+13.5k^2+2G-2.5k+3)+\lambda d_{in},\\
    &\mathcal{F}_\mathrm{PowerMLP}=4 d_{in} d_{out}+(k-1) d_{out}+\lambda d_{in}
\end{align*}}
\hskip-3pt
where $d_{in}$ and $d_{out}$ denote the input and output dimensions of the layer. The KAN layer uses $(k,G)$-grid, and the PowerMLP layer is $k$-order. Details refer to Appendix A.

Then we compare FLOPs of three network layers with same number of parameters, given by the ratio of FLOPs $\mathcal{F}$ to numbers of parameters $\mathcal{N}$:
{\small
\begin{align*}
    &r_\mathrm{MLP}=\frac{2 d_{in} d_{out}}{d_{in} d_{out}+d_{out}},\\
    &r_\mathrm{KAN}=\frac{d_{in}d_{out}(9kG+13.5k^2+2G-2.5k+3)+\lambda d_{in}}{d_{in}d_{out}(k+G+2)},\\
    &r_\mathrm{PowerMLP}=\frac{4 d_{in} d_{out}+(k-1) d_{out}+\lambda d_{in}}{2d_{in} d_{out}+d_{out}}.
\end{align*}}

Assuming that $d_{in}$ and $d_{out}$ increase at the same rate, $r_\mathrm{MLP}$ and $r_\mathrm{PowerMLP}$ tend towards values less than $2$, while $r_\mathrm{KAN}$ approaches to  numbers larger than $20$ for $k\geq3,G\geq3$ ($k=3,G=3$ are the smallest values used in KAN paper).
{\bf Thus, PowerMLP shares a close computing speed with MLP, and is over 10 times faster than KAN under the FLOPs metric.}
We give an example of FLOPs of KAN, MLP and PowerMLP with almost the same number of parameters in Table \ref{tab:flops}. Training speed comparisons in experimental settings are in Section \ref{sec:speed}.
\begin{table}[ht]
    \small
    \centering
    \begin{tabular}{c c c c}
        \hline
        Network & Shape & \#Params & FLOPs \\
        \hline
        KAN ($G=3,k=3$) & $[2,1,1]$ & $24$ & $564$ \\
        MLP (ReLU) & $[2,6,1]$ & $25$ & $36$ \\
        PowerMLP ($k=3$) & $[2,4,1]$ & $25$ & $40$ \\
        \hline
    \end{tabular}
    \captionof{table}{Comparison of KAN, MLP and PowerMLP. With almost the same parameters, MLP and PowerMLP have much fewer FLOPs than KAN.
    }
    \label{tab:flops}
\end{table}

\section{Experiments}
\label{sec-exp}
In Section \ref{sec-theory}, PowerMLPs are shown to define the same function space as KANs over bounded intervals and achieve faster computation.
In this section, we employ several experiments to validate these theoretical findings and demonstrate the advantages of PowerMLP.

Four experiments are conducted.
(1) We considered AI for science tasks in the KAN paper~\cite{liu2024kan}  under the same settings, showing that PowerMLP performs better.
(2) More complex tasks like machine learning, natural language processing, and image classification are considered. We show that PowerMLP outperforms KAN in all tasks.
(3) We compared training time and convergence time of KAN and PowerMLP, validating that PowerMLP can be much faster.
(4) We conducted an ablation experiment to show that both the basis function and ReLU-$k$ activation are needed for the performance of PowerMLP.

All KANs in the paper are the latest version (0.2.5) up to 2024-8-14.
More details are given in Appendix B.\footnote{Code is available at https://github.com/Iri-sated/PowerMLP.}

\subsection{AI for Science Tasks}
\label{sec:ai4sci}
\subsubsection{Function Fitting}
PowerMLP was tested on a regression task for $16$ special functions from KAN's experiments~\cite{liu2024kan}. For KAN, MLP, and PowerMLP, we choose two sizes for the networks:
\textbf{(1) Small size: } KAN with $k=3,G=3$, shape $[2,1,1]$, and $24$ parameters; MLP with shape $[2,6,1]$ and $25$ parameters; $3$-order PowerMLP with shape $[2,4,1]$ and $25$ parameters;
\textbf{(2) Large size:} KAN with $k=3,G=100$, shape $[2,2,1,1]$ and $735$ parameters; MLP with shape $[2,32,18,1]$ and $709$ parameters; $3$-order PowerMLP with shape $[2,32,8,1]$ and $689$ parameters.

Results are given in Table \ref{tab:specialfunc}.
We see that PowerMLP achieves the best results on 11/10 (small size/large size) out of 16 cases. This is attributed to that PowerMLP has ReLU-$k$ for stronger power of expression than MLP and can better converge than KAN with a simpler structure of non-parameter activations.

\begin{table*}[ht]
    \small
    \centering
    \begin{tabular}{c | c c c | c c c}
        \hline
        \multirow{2}{*}{\makecell{Function\\Name}} & \multicolumn{3}{c|}{Small Size: $\sim25$ parameters} & \multicolumn{3}{c}{Large Size: $\sim700$ parameters}\\
         & KAN & MLP & PowerMLP & KAN & MLP & PowerMLP \\
        \hline
        { JE} & $4.63\times10^{-3}$ & $4.98\times 10^{-3}$ & $\mathbf{5.79\times 10^{-4}}$ & $1.04\times10^{-4}$ & $5.88\times 10^{-4}$ & $\mathbf{7.23\times 10^{-5}}$ \\
        { IE1} & $1.34\times10^{-2}$ & $5.79\times 10^{-3}$ & $\mathbf{3.43\times 10^{-3}}$ & ${4.52\times10^{-5}}$ & $5.57\times 10^{-4}$ & $\mathbf{3.37\times 10^{-5}}$ \\
        { IE2} & $1.16\times10^{-2}$ & $4.71\times 10^{-3}$ & $\mathbf{1.73\times 10^{-3}}$ & $1.18\times10^{-3}$ & $5.98\times 10^{-4}$ & $\mathbf{3.20\times 10^{-5}}$ \\
        { B1} & $7.71\times10^{-1}$ & $3.94\times 10^{-2}$ & $\mathbf{3.93\times 10^{-2}}$ & $1.70\times10^{-2}$ & $5.47\times 10^{-3}$ & $\mathbf{2.77\times 10^{-3}}$ \\
        { B2} & $7.94\times10^{-2}$ & $\mathbf{6.02\times 10^{-2}}$ & $7.75\times 10^{-2}$ & $\mathbf{1.77\times10^{-3}}$ & $4.62\times 10^{-3}$ & $2.61\times 10^{-3}$ \\
        { MB1} & $2.29\times10^{0}$ & ${3.76\times10^{-2}}$ & $\mathbf{3.48\times 10^{-2}}$ & $1.70\times10^{-2}$ & $5.11\times10^{-3}$ & $\mathbf{3.75\times 10^{-3}}$ \\
        { MB2} & $7.97\times10^{-1}$ & $1.04\times10^{-2}$ & $\mathbf{7.47\times 10^{-3}}$ & ${9.44\times10^{-5}}$ & $1.01\times10^{-3}$ & $\mathbf{7.42\times 10^{-5}}$ \\
        { AL $(m = 0)$} & $1.09\times10^{-1}$ & $8.14\times10^{-2}$ & $\mathbf{7.21\times 10^{-2}}$ & $\mathbf{1.88\times10^{-3}}$ & $7.45\times10^{-3}$ & $6.72\times 10^{-3}$ \\
        { AL $(m = 1)$} & $1.25\times10^{-1}$ & $7.49\times10^{-2}$ & $\mathbf{6.94\times 10^{-2}}$ & $1.40\times10^{-2}$ & $1.20\times10^{-2}$ & $\mathbf{1.02\times 10^{-2}}$ \\
        { AL $(m = 2)$} & $2.30\times10^{-1}$ & $1.10\times10^{-1}$ & $\mathbf{9.95\times 10^{-2}}$ & $\mathbf{1.67\times10^{-3}}$ & $8.83\times10^{-3}$ & $7.29\times 10^{-3}$ \\
        { SH $(m = 0, n = 1)$} & ${4.05\times10^{-5}}$ & $2.02\times10^{-3}$ & $\mathbf{2.59\times10^{-5}}$ & $8.39\times10^{-6}$ & $1.45\times10^{-4}$ & $\mathbf{7.77\times10^{-6}}$ \\
        { SH $(m = 1, n = 1)$} & $1.92\times10^{-2}$ & $\mathbf{5.57\times10^{-3}}$ & $1.20\times10^{-2}$ & $\mathbf{7.03\times10^{-5}}$ & $4.16\times10^{-4}$ & $2.60\times10^{-4}$ \\
        { SH $(m = 0, n = 2)$} & $\mathbf{5.94\times10^{-5}}$ & $4.46\times10^{-3}$ & $4.79\times10^{-4}$ & $\mathbf{1.02\times10^{-5}}$ & $3.17\times10^{-4}$ & $2.55\times10^{-5}$ \\
        { SH $(m = 1, n = 2)$} & $3.25\times10^{-2}$ & $1.81\times10^{-2}$ & $\mathbf{2.37\times10^{-3}}$ & $2.27\times10^{-3}$ & $8.90\times10^{-4}$ & $\mathbf{1.45\times10^{-4}}$ \\
        { SH $(m = 2, n = 2)$} & $\mathbf{1.17\times10^{-2}}$ & $1.59\times10^{-2}$ & $2.16\times10^{-2}$ & $1.09\times10^{-2}$ & $8.32\times10^{-2}$ & $\mathbf{3.66\times10^{-5}}$ \\
        \hline
    \end{tabular}
    \caption{Fitting Special Functions.
    JE: Jacobian elliptic functions,
    IE1/IE2: Incomplete elliptic integral of the first/second kind,
    B1/B2: Bessel function of the first/second kind,
    MB1/MB2: Modified Bessel function of the first/second kind,
    AL: Associated Legendre function,
    SH: spherical harmonics.
    All the values are test RMSE loss, the less the better. The best results are marked as boldface.
    PowerMLP achieves the best results on 11/10 (small/large size) out of 16 cases.
    }
    \label{tab:specialfunc}
\end{table*}

\subsubsection{Knot Theory}
\begin{figure}[!t]
    \centering
        \includegraphics[width=0.9\linewidth]{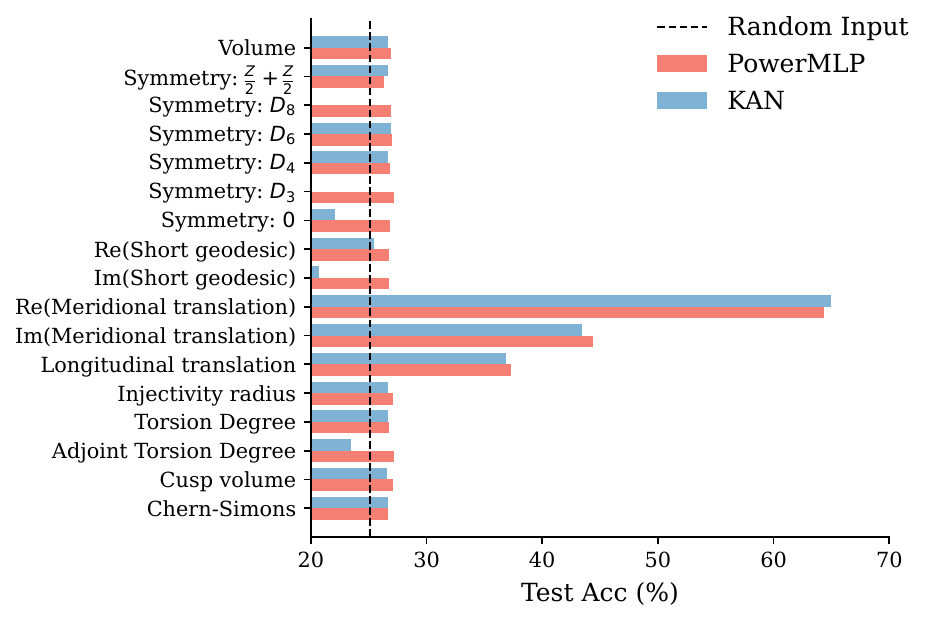}\\
        \includegraphics[width=.9\linewidth]{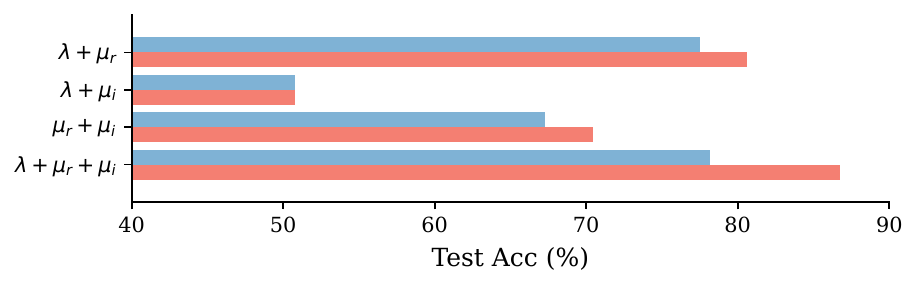}
    \caption{In the upper figure, PowerMLP can correctly find that $3$ of $17$ geometric invariants have influence on the output. Additionally, PowerMLP outperforms KAN in $15$ of $17$ input cases while KAN fails to converge with Symmetry $D_3$ or $D_8$ as input. In the bottom figure, trained on part or all of the $3$ influencing geometric invariants, PowerMLP achieves much higher test accuracy than KAN in $3$ cases.
    }
    \label{fig:knotsacc}
\end{figure}

\citet{davies2021advancing} used MLP to discover highly non-trivial relations among $17$ \textit{geometric invariants} and \textit{signatures} in knot theory. With $17$ geometric invariants as inputs and the corresponding signature as output, they trained MLPs and found a strong connection between signatures and $3$ of all $17$ geometric invariants: the \textit{longitudinal translation $\lambda$, real and image part of {meridional translation} $\mu_r,\mu_i$}.
Based on this observation, they proved a theorem and explained the connection theoretically, which is an interesting example of AI for math tasks~\cite{davies2022signature}.

\citet{liu2024kan} reproduced the same experiment with KANs. Following their settings, we conduct the experiment with our PowerMLP. In Figure \ref{fig:knotsacc}, we show the test accuracy of using a single geometric invariant as input each time to predict the signature. Most geometric invariants are close to random inputs, but the longitudinal translation, real and image part of the meridional translation show relationships with the signature. For better comparison, we also show results of a KAN with the same depth and almost same number of parameters. KAN performs worse than PowerMLP in $15$ of $17$ geometric invariants. In particular, KAN fails to converge with Symmetry $D_3$ or $D_8$ as input.

Furthermore, in Figure \ref{fig:knotsacc}, we show test accuracy of PowerMLPs trained on all or part of the three relevant geometric invariants $\lambda,\,\mu_i$ and $\mu_r$. With a test accuracy of $86.74\%$, PowerMLP successfully validates the connection among signature and $\lambda,\,\mu_i$, $\mu_r$. Furthermore, with a test accuracy of $80.62\%$, PowerMLP can also find the connection among $\lambda,\mu_r$ and the signature discovered by KAN. More importantly, compared to the close test accuracy of KAN and PowerMLP in single geometric invariant input, PowerMLP achieves much higher test accuracy than KAN in $3$ out of $4$ cases of combination input. This indicates that PowerMLP can better utilize the correlation between inputs.

\subsection{More Complex Tasks}
\begin{figure}[t]
    \centering
    \includegraphics[width=0.90\linewidth]{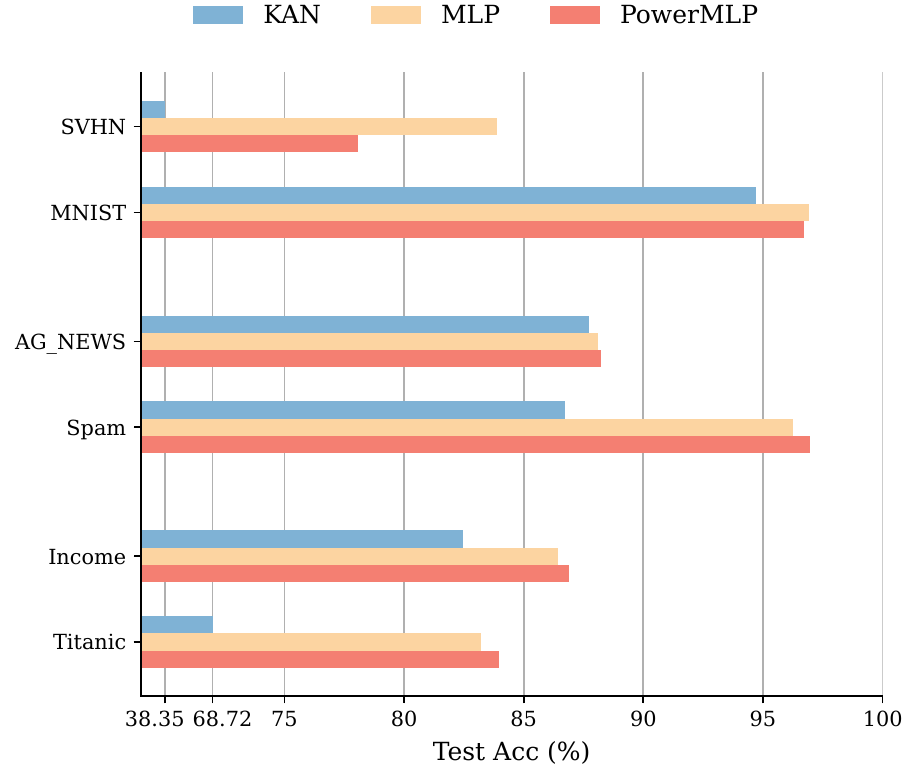}
    \caption{Test accuracy of three networks on multiple classification tasks.}
    \label{fig:clsacc}
\end{figure}
\label{sec:complex}
We perform three more complex tasks.
We use three networks with approximately the same number of parameters.

\subsubsection{Machine Learning}
For basic machine learning tasks, we conduct two experiments on Titanic and Income~\cite{becker196adult}, which are classification tasks of small input dimension.
In Figure \ref{fig:clsacc}, we show the test accuracy of three networks, and PowerMLP outperforms MLP and KAN.

\subsubsection{Natural Language Processing}
We conduct two experiments on SMS Spam Collection (Spam)~\cite{gomez2006content} and AG\_NEWS~\cite{zhang2015character} dataset, which are text classification tasks. We use TF-IDF transformation~\cite{ramos2003using} to convert text into vectors, and \textit{networks need to deal with high-dimensional sparse inputs}.
In Figure \ref{fig:clsacc}, we show the test accuracy of three networks, and PowerMLP outperforms MLP and KAN.

\subsubsection{Image Classification}
For image classification tasks, we conduct two experiments on MNIST~\cite{lecun1998gradient} and SVHN~\cite{netzer2011reading} datasets. We convert SVHN data to greyscale images, and flatten image tensor into $1$-dimension vectors. In this task, \textit{networks need to deal with high-dimensional inputs with strong connection among different input dimensions}.
In Figure \ref{fig:clsacc}, we show the test accuracy of three networks, and PowerMLP outperforms KAN.

In summary, with almost the same number of parameters and the same depth, PowerMLP achieves better accuracy than KAN in all tasks.
PowerMLP outperforms MLP in machine learning (Income, Titanic) and language processing (AG\_NEWS, Spam), while performing worse in image classification (SVHN, MNIST).

\subsection{Training Time}
\label{sec:speed}
In Table \ref{tab:speed}, eight tasks are considered. The first five tasks are function regression tasks in Section \ref{sec:ai4sci}, and the last three tasks are machine learning, language processing, and image classification tasks in Section \ref{sec:complex}.
For better comparison, the experiments are on a single NVIDIA GeForce RTX 4090 GPU, repeated each task $10$ times to take an average, and networks in each task are trained
with the same hyperparameters.
From Table \ref{tab:speed}, among all tasks and in different numbers of parameters, the training time of PowerMLP is close to MLP, which is about 40 times less than KAN. This is consistent with our theoretical analysis in Section \ref{sec-FLOP}.

\begin{table}[!t]
    \centering
    \small
    \begin{tabular}{c c c c c}
        \hline
        \multirow{2}{*}{Tasks} & \multirow{2}{*}{\#Params} & \multicolumn{3}{c}{Time(s)}\\
        & & KAN & MLP & PowerMLP\\
        \hline
        {JE} & \textasciitilde 25 & 17.73 & 0.583 & 0.738 \\
        {IE1} & \textasciitilde 25 & 25.32 & 0.430 & 0.674 \\
        {IE2} & \textasciitilde 25 & 31.91 & 0.584 & 0.709 \\
        {B1} & \textasciitilde 25 & 29.68 & 0.422 & 0.646 \\
        {B2} & \textasciitilde 25 & 35.20 & 0.559 & 0.696 \\
        Titanic & \textasciitilde 100 & 35.07 & 0.529 & 0.591 \\
        Spam & \textasciitilde 800 & 62.06 & 0.568 & 0.655 \\
        SVHN & \textasciitilde $1.3\times10^5$ & 89.82 & 1.53 & 2.39\\
        \hline
    \end{tabular}\\
    \caption{Training times on 8 tasks.
    Training times of PowerMLP are about 40 times smaller than KAN on average.}
    \label{tab:speed}
\end{table}
To be more comprehensive, in Figure \ref{fig:convergence}, we give the variation curves of the test RMSE loss relative to the training time in one training progress. Two curves with similar colors represent two training processes of the same network. We can see that MLP and PowerMLP also converge much faster than KAN in terms of training time.

\begin{figure}[!t]
    \centering
    \includegraphics[width=0.9\linewidth]{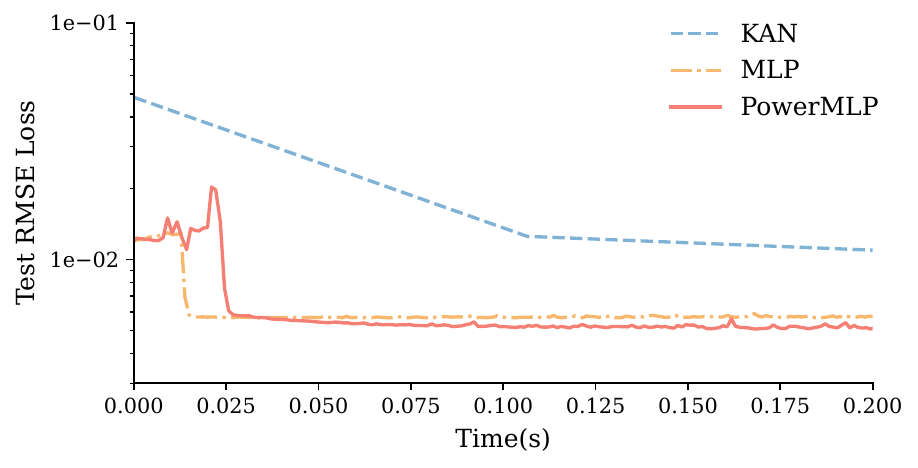}
    \caption{Time of convergence. MLP and PowerMLP converge much faster than KAN.}
    \label{fig:convergence}
\end{figure}

\begin{figure}[!t]
    \centering
    \includegraphics[width=0.9\linewidth]{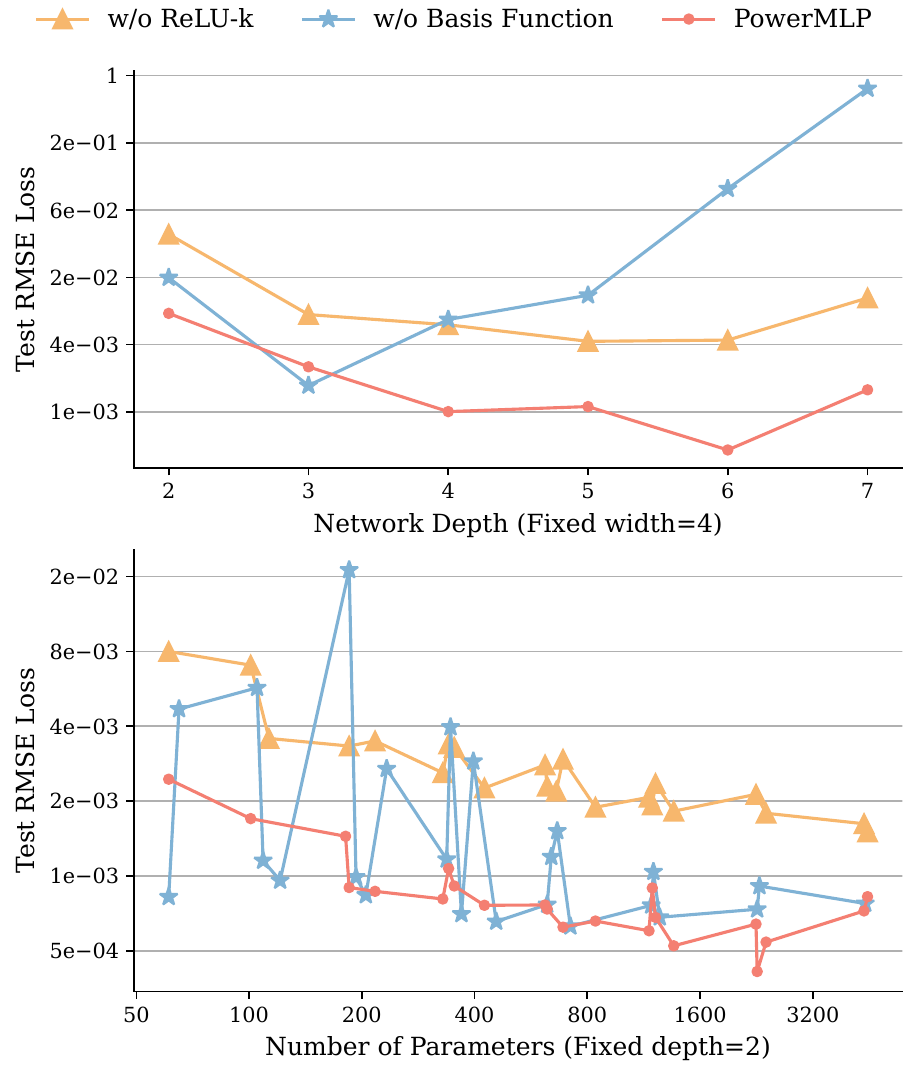}
    \captionof{figure}{Ablation study.
    Basis function enhances training stability, while ReLU-$k$ improves expressive ability.}
    \label{fig:ablation}
\end{figure}

\subsection{Ablation Study}
We show that both the basis function and the ReLU-$k$ are useful in PowerMLP.
In Figure \ref{fig:ablation}, we compare three networks, PowerMLP, PowerMLP without basis function, and PowerMLP without ReLU-$k$ activation (ReLU instead) on a function regression task for $f(x,y)=x\exp(-y)$.

The upper graph shows the RMSE loss on the test set with variant network depth for fixed width of $4$.
Although PowerMLPs without basis function perform well in shallow networks, they fail to converge for large depth.
The bottom graph shows the influence of the number of parameters, with a fixed network depth of $2$.
PowerMLPs without basis function achieve nearly the same test RMSE loss with PowerMLP for some networks, while performing much worse in other cases, indicating that basis function enhances training stability. In all situations, PowerMLPs perform better than PowerMLPs without ReLU-$k$, consistent with ReLU-$k$'s better expressive ability.

\section{Conclusion}
In this paper, we introduce a novel neural network architecture, PowerMLP, which employs the powers of ReLU as activation and expresses linear combination as inner product to represent spline function.
PowerMLP can be viewed as a more efficient and stronger version of KAN.
In terms of expressiveness, PowerMLPs define a larger or equal function space than KANs.
In terms of efficiency, PowerMLPs can be trained with over 10 times fewer FLOPs than KANs.
We conducted comprehensive experiments to demonstrate that PowerMLPs achieve higher accuracy in most cases and can be trained about 40 times faster compared to KANs.

{\bf Limitations and Future Work.}
A limitation of PowerMLP is that it fails to achieve exceptional performance in complicated computer vision tasks and long-text processing.
This stems from PowerMLP's relatively simple architecture, which lacks specialized structures such as convolutional layers or attention mechanisms.
However, given that PowerMLP shares a similar underlying architecture with traditional MLP, it is feasible to substitute the MLP parts in CNNs or Transformers with PowerMLP naturally and expect that more complicated problems can be solved better by using PowerMLP as a basic building block.

\newpage
\section*{Acknowledgements}
This work is partially supported by
NKRDPC grant 2018YFA0704705, NSFC grant 12288201,
and the Beijing Natural Science Foundation grant QY23156.


\appendix
\onecolumn
\section{Proofs of Section 4}
\subsection{Proofs of Section 4.2}

\begin{lemma*}[Lemma \ref{lem:bspline}, Represent the B-spline with powers of ReLU]
If $t_u \neq t_v (\forall u \neq v)$, then the $k$-order B-spline on the knot sequence $t=(t_j, \cdots, t_{j+k+1})$ can be represented as a linear combination of $\sigma_{k}$ functions:
    \begin{align}
    \label{apd:eq-Bs}
        B_{j,k,t}(x)&=\sum_{i=j}^{j+k+1}\frac{t_{j+k+1}-t_j}{\prod_{l=j}^{j+k+1,l\neq i}(t_l-t_i)}\sigma_{k}(x-t_i).
    \end{align}
\end{lemma*}
\begin{proof}
We prove the lemma by induction on $k$.
    When $k=0$, it can be verified as follows
    \begin{equation*}
        \hbox{LHD}=\sigma_0(x-t_j)-\sigma_0(x-t_{j+1})=\left\{\begin{aligned}
            &1,\quad&t_j\leq x<t_j+1\\
            &0,&\mathrm{otherwise}
        \end{aligned}\right.=\hbox{RHD}.
    \end{equation*}
    Note that $\sigma_0(t)$ is the binary step function which equals $1$ for $t\geq 0$ and $0$ otherwise. Therefore, when $k=0$, the lemma holds.

Assume that the Eq. \eqref{apd:eq-Bs} holds for $k$.
According to Eq. \eqref{eq:recur} (Refer to the main text of the paper), the $k+1$-order B-spline satisfies
\begin{align*}
    &B_{j,k+1,t}(x)\\
    &=\frac{x-t_j}{t_{j+k+1}-t_j}B_{j,k,t}(x)+\frac{t_{j+k+2}-x}{t_{j+k+2}-t_{j+1}}B_{j+1,k,t}(x)\\
    &=\frac{1}{\prod_{l=j+1}^{j+k+1}(t_l-t_j)}(x-t_j)\sigma_{k}(x-t_j)+\frac{1}{\prod_{l=j+1}^{j+k+1}(t_l-t_{j+k+2})}(t_{j+k+2}-x)\sigma_{k}(x-t_{j+k+2})\\
    &\quad+\sum_{i=j+1}^{j+k+1}\left[\frac{1}{\prod_{l=j}^{j+k+1,l\neq i}(t_l-t_i)}(x-t_j)+\frac{1}{\prod_{l=j+1}^{j+k+2,l\neq i}(t_l-t_i)}(t_{j+k+2}-x)\right]\sigma_{k}(x-t_i)\\
    &=\frac{t_{j+k+2}-t_j}{\prod_{l=j+1}^{j+k+2}(t_l-t_j)}\sigma_{k+1}(x-t_j)+\frac{t_{j+k+2}-t_j}{\prod_{l=j}^{j+k+1}(t_l-t_{j+k+2})}\sigma_{k+1}(x-t_{j+k+2})\\
    &\quad+\sum_{i=j+1}^{j+k+1}\left[\left(\frac{1}{\prod_{l=j}^{j+k+1,l\neq i}(t_l-t_i)}-\frac{1}{\prod_{l=j+1}^{j+k+2,l\neq i}(t_l-t_i)}\right)(x-t_i)\right.\\
    &\quad\left.+\frac{t_j-t_i}{\prod_{l=j}^{j+k+1,l\neq i}(t_l-t_i)}+\frac{t_{j+k+2}-t_i}{\prod_{l=j+1}^{j+k+2,l\neq i}(t_l-t_i)}\right]\sigma_{k}(x-t_i)\\
    &=\frac{t_{j+k+2}-t_j}{\prod_{l=j+1}^{j+k+2}(t_l-t_j)}\sigma_{k+1}(x-t_j)+\frac{t_{j+k+2}-t_j}{\prod_{l=j}^{j+k+1}(t_l-t_{j+k+2})}\sigma_{k+1}(x-t_{j+k+2})\\
    &\quad+\sum_{i=j+1}^{j+k+1}\frac{t_{j+k+2}-t_j}{\prod_{l=j}^{j+k+2,l\neq i}(t_l-t_i)}\sigma_{k+1}(x-t_i)\\
    &=\sum_{i=j}^{j+k+2}\frac{t_{j+k+2}-t_j}{\prod_{l=j}^{j+k+2,l\neq i}(t_l-t_i)}\sigma_{k+1}(x-t_i).
\end{align*}
    %
    Thus, the Eq. \eqref{apd:eq-Bs} holds for $k+1$ and the lemma is proved.
\end{proof}

\begin{theorem*}[Theorem \ref{th-k4p}, KAN is a Subset of PowerMLP]
    Fix the input dimension of networks as $n$. Let $\mathcal{K}_{d,w,k,G,p}$ be the set of all KAN networks with depth $d$, width $w$, $p$ nonzero parameters, using $(k,G)$-spline; and $\mathcal{P}_{d,w,k,p}$ be the set of all PowerMLPs with depth $d$, width $w$, $k$-th power of ReLU  and $p$ nonzero parameters. Then it holds
    \begin{equation*}
        \mathcal{K}_{d,w,k,G,p}\subset\mathcal{P}_{d,w^2(G+k),k,p}.
    \end{equation*}
\end{theorem*}
\begin{proof}
    Since both KAN and PowerMLP are composition of layers, it suffices to prove that one layer of KAN can be represented by one layer of PowerMLP with same number of nonzero parameters.
    Denote the input vector
    $\mathbf{x}=(x_1,\cdots,x_n)^\top\in\mathbf{R}^n$
    and the output vector $\mathbf{z}=\Phi(\mathbf{x})\in\mathbb{R}^m$ where $n=n_\ell,\,m=n_{\ell+1}$.
    For convenience, we hide the subscript $\ell$ which connects to the layer.
    Then it is equivalent to show that for all KAN layer, there exist $\alpha,\,\beta,\,\omega,\,\gamma$ satisfying
    \begin{equation}
    \label{eq:rrn2kan}
        \begin{pmatrix}
			\sum_{p=1}^{n}\left[u_{1,p} b(x_p)+v_{1,p}\mathrm{spline}_{k,G_{1,p}}(x_p)\right]\\
			\vdots\\
			\sum_{p=1}^{n}\left[u_{m,p} b(x_p)+v_{m,p}\mathrm{spline}_{k,G_{m,p}}(x_p)\right]
		\end{pmatrix}=\alpha b(\mathbf{x})+\beta\sigma_{k-1}(\omega \mathbf{x}+\gamma)\qquad\mathrm{for\,}\mathbf{x}\in\mathbb{R}^{n}.
    \end{equation}

    \textbf{Step 1: Basis Function}

    Simply select $\alpha=(u_{j,k})_{1\leq j\leq m,1\leq k\leq n}$, then we have that
    \begin{equation*}
        \begin{pmatrix}
			\sum_{p=1}^{n}u_{1,p}b(x_p)\\
			\vdots\\
			\sum_{p=1}^{n}u_{m,p}b(x_p)
		\end{pmatrix}=\alpha b(\mathbf{x}),
    \end{equation*}
    which proves the part of the Eq. \eqref{eq:rrn2kan} about the basis function.

    \textbf{Step 2: Spline Function}

    We first express one spline function $\mathrm{spline}_{k,G_{q,p}}$ in matrix form of $\sigma_{k}$ functions.
    Denote the coefficient of the spline function before the j-th B-spline $B_{j,k,t_{q,p}}$ as $c^j_{q,p}$ (refer to Eq. \eqref{eq-sfun} in the main text of the paper)) and denote coefficients before $\sigma_k(x-t_i)$ in Eq. \eqref{apd:eq-Bs} as $\xi^i_{j,t}$. Then we have
    \begin{equation*}
        \mathrm{spline}_{k,G_{q,p}}(x_p)=\overline{\beta}_{q,p}\sigma_{k}(\overline{\omega}_{q,p}\mathbf{x}+\overline{\gamma}_{q,p}),
    \end{equation*}
    where $\overline{\beta}_{q,p},\,\overline{\omega}_{q,p}$, and $\overline{\gamma}_{q,p}$ are
    {\footnotesize
    \begin{align*}
        &\overline{\beta}_{q,p}=
        \begin{pmatrix}
            \sum_{i=-k}^{-k} c^i_{q,p}\xi^{-k}_{i,t_{q,p}} &
            \sum_{i=-k}^{-k+1} c^i_{q,p}\xi^{-k+1}_{i,t_{q,p}} &
            \cdots &
            \sum_{i=-k}^{0} c^i_{q,p}\xi^{0}_{i,t_{q,p}} &
            \cdots &
            \sum_{i=G-k-1}^{G-1} c^{i}_{q,p}\xi^{G-1}_{i,t_{q,p}}
        \end{pmatrix},
        \\
        &\overline{\omega}_{q,p}=\begin{pmatrix}
            0 & \cdots & 0 & 1 & 0 & \cdots & 0\\
            0 & \cdots & 0 & 1 & 0 & \cdots & 0\\
            \vdots & & \vdots & \vdots & \vdots & & \vdots\\
            0 & \cdots & 0 & 1 & 0 & \cdots & 0
        \end{pmatrix}\mathrm{\,(only\,the\,p-th\,column\,being\,1)},
        \quad
        \overline{\gamma}_{q,p}=-\begin{pmatrix}
            t^{-k}_{p,q} & t^{-k+2}_{p,q} & \cdots & t^{G-1}_{p,q}
        \end{pmatrix}^\top.
    \end{align*}}

    Then the spline part of the equation \eqref{eq:rrn2kan} can be calculated:
    {
    \begin{align*}
&      \quad  \begin{pmatrix}
            \sum_{p=1}^{n}v_{1,p}\mathrm{spline}_{k,G_{1,p}}(x_p)\\
            \vdots\\
            \sum_{p=1}^{n}v_{m,p}\mathrm{spline}_{k,G_{m,p}}(x_p)
        \end{pmatrix}\\
        &=
        \begin{pmatrix}
            v_{11} & \cdots & v_{1n} & 0 & \cdots & 0 & 0 & \cdots & 0 \\
            \vdots & & \vdots & \vdots & & \vdots & \vdots & & \vdots \\
            0 & \cdots & 0 & 0 & \cdots & 0 & v_{m1} & \cdots & v_{mn}
        \end{pmatrix}\begin{pmatrix}
            \mathrm{spline}_{k,G_{1,1}}(x_1)\\
            \vdots\\
            \mathrm{spline}_{k,G_{1,n}}(x_n)\\
            \vdots\\
            \mathrm{spline}_{k,G_{m,n}}(x_n)
        \end{pmatrix}\\
        &=\begin{pmatrix}
            v_{11}\overline{\beta}_{1,1} & \cdots & v_{1n}\overline{\beta}_{1,n} & 0 & \cdots & 0 & 0 & \cdots & 0 \\
            \vdots & & \vdots & \vdots & & \vdots & \vdots & & \vdots \\
            0 & \cdots & 0 & 0 & \cdots & 0 & v_{m1}\overline{\beta}_{m,1} & \cdots & v_{mn}\overline{\beta}_{m,n}
        \end{pmatrix}
        \sigma_{k}\left(\begin{pmatrix}
            \overline{\omega}_{1,1}\\
            \vdots\\
            \overline{\omega}_{1,n}\\
            \vdots\\
            \overline{\omega}_{m,n}
        \end{pmatrix}\mathbf{x}+\begin{pmatrix}
            \overline{\gamma}_{1,1}\\
            \vdots\\
            \overline{\gamma}_{1,n}\\
            \vdots\\
            \overline{\gamma}_{m,n}
        \end{pmatrix}\right)\\
        &\doteq\beta\sigma_{k}(\omega \mathbf{x}+\gamma),
    \end{align*}}
\hskip-4pt
where $\alpha\in\mathbb{R}^{m\times n},\, \beta\in\mathbb{R}^{m\times mn(G+k)},\, \omega\in\mathbb{R}^{mn(G+k)\times n},\, \gamma\in\mathbb{R}^{mn(G+k)\times 1}$.

Note that the number of parameters in $\alpha$ is equal to $mn$. $\overline{\beta}_{q,p}$ has $G+k$ nonzero parameters, thus $\beta$ has $mn(G+k)$ parameters. $\omega$ is a $0-1$ matrix that depends only on the number of grids in spline function, without any trainable parameters. $\overline{\gamma}_{q,p}$ consists of $G+k$ parameters from the knot sequence of each spline function, which is also set to be untrained in KAN. To sum up, the left-hand side of Eq. \ref{eq:rrn2kan} has $mn(G+k)$ nonzero trainable parameters, same as the right-hand side.

Notice that $\beta$ can be absorbed into the weight in next layer, so a KAN layer can be written as a PowerMLP layer. Since both KAN and PowerMLP are composition of layers, we prove that a KAN of $w$ width can be represented as a PowerMLP with same order, depth, number of nonzero parameters, and width $w^2(G+k)$.

\end{proof}

\subsection{Proofs of Section 4.3}
\begin{lemma*}[Lemma \ref{lem:affine}, Affine Transformations]\label{apd:affine}
    Consider an affine transformation on $\mathbb{R}$:
 $       \mathcal{A}(x)=\omega x+\gamma$.
For any $G$, we can find a $(k,G)$-grid
$t=(\underline{t_{-k},\cdots,t_{-1}},t_0,t_1,\cdots,t_G,$
$\underline{t_{G+1},\cdots,t_{G+k}}),$
    and a $k$-order spline function
\begin{equation}
    \label{apd-eq:affine}
    \begin{aligned}
        &\mathrm{spline}_{k,G}(x)=\sum_{j=-k}^{G-1}c_jB_{j,k,t}(x),\\
    \end{aligned}
    \end{equation}
where $c_j=\left(\sfrac{\sum_{i=j+1}^{j+k}t_i}{k}\right)\omega+\gamma,\,k>0$,
     such that $\mathcal{A}(x)=\mathrm{spline}_{k,G}(x)$ for $t_0\le x\le t_G$.
\end{lemma*}
\begin{proof}
    We prove the lemma by induction on $k$. When $k=1$, the right-hand side can be calculated as
    \begin{equation*}
        \sum_{j=-1}^{G-1}c_jB_{j,1,t}(x)=\sum_{j=-1}^{G-1}\left(t_{j+1}\omega+\gamma\right)B_{j,1,t}(x).
    \end{equation*}
    We check on an interval $[t_j,t_{j+1}]$ for $\forall t_j(0\leq j\leq {G-1})$. Based on the definition of B-spline, only $B_{j-1,1,t}$ and $B_{j,1,t}$ are nonzero on $[t_j,t_{j+1}]$. Thus, we have
    \begin{equation*}
        \sum_{j=-1}^{G-1}c_jB_{j,1,t}(x)=(t_j\omega+\gamma)\frac{t_{j+1}-x}{t_{j+1}-t_j}+(t_{j+1}\omega+\gamma)\frac{x-t_j}{t_{j+1}-t_j}=\omega x+\gamma,\quad\mathrm{for\,}x\in [t_j,t_{j+1}].
    \end{equation*}
    Since it holds for all $j=0,1,\cdots,G-1$, the case where $k=1$ is proved.

    Assume that the lemma holds for order $k-1$. According to Definition \ref{def:bspline}, for $t_0\leq x\leq t_G$, the $k$ case satisfies
    \begin{align*}
        &\sum_{j=-k}^{G-1}c_jB_{j,k,t}(x)\\
        &=\sum_{j=-k+1}^{G-1}\left[c_j\frac{x-t_j}{t_{j+k}-t_j}+c_{j-1}\frac{t_{j+k}-x}{t_{j+k}-t_j}\right]B_{j,k-1,t}(x)\\
        &=\sum_{j=-k+1}^{G-1}\left[\frac{x-t_j}{t_{j+k}-t_j}\left(\frac{t_{j+1}+\cdots+t_{j+k}}{k} \omega+\gamma\right)+\frac{t_{j+k}-x}{t_{j+k}-t_j}\left(\frac{t_{j}+\cdots+t_{j+k-1}}{k} \omega+\gamma\right)\right]B_{j,k-1,t}(x)\\
        &=\sum_{j=-k+1}^{G-1}\left[\frac{t_{j+1}+\cdots+t_{j+k-1}+x}{k}\omega+\gamma\right]B_{j,k-1,t}(x)\\
        &=\frac{k-1}{k}\sum_{j=-k+1}^{G-1}\left[\frac{t_{j+1}+\cdots+t_{j+k-1}}{k-1}\omega+\gamma\right]B_{j,k-1,t}(x)+\frac{1}{k}(\omega x+\gamma)\sum_{j=-k+1}^{G-1}B_{j,k-1,t}(x).
    \end{align*}
    The first term is $\frac{k-1}{k}(\omega x+\gamma)$ according to the inductive hypothesis. \citet{deBoor1978practical} shows in B-spline property (iv) that $\sum_{j=-k+1}^{G-1}B_{j,k-1,t}(x)=1$ for $t_0\leq x\leq t_G$. The equation can be simplified as
    \begin{equation*}
        \frac{k-1}{k}(\omega x+\gamma)+\frac{1}{k}(\omega x+\gamma)=\omega x+\gamma.
    \end{equation*}
    Thus, the formula holds for $k$ and the lemma is proved.
\end{proof}

\begin{lemma*}[Lemma \ref{lem:repu}, ReLU-$k$ function]\label{apd:reluk}
We can find a $(k,2)$-grid $t=(\underline{t_{-k},\cdots,t_{-1}},t_0,0,t_2,\underline{t_3,\cdots,t_{k+2}})$ and a $k$-order spline function defined on $t$
    \begin{equation*}
        \mathrm{spline}_{k,2}(x)=\sum_{j=-k}^{1}\left[\left(\prod_{l=j+1}^{j+k}\sigma_1(t_l)\right) B_{j,k,t}(x)\right],
    \end{equation*}
such that $\sigma_{k}(x)=\mathrm{spline}_{k,2}(x)$ for $t_0\le x\le  t_2$.
\end{lemma*}
\begin{proof}
    Note that $\sigma_1$ is actually the ReLU function. For $l=-k+1,\cdots,0,1$, condition $t_l\leq0$ infers $\sigma_1(t_l)=0$, so we have
    \begin{equation*}
        \mathrm{spline}_{k,2}(x)=\left(\prod_{l=2}^{k+1}t_l\right) B_{1,k,t}(x).
    \end{equation*}
    According to Lemma \ref{lem:bspline}, we convert B-spline into a linear combination of $\sigma_k$. Since $t_0\leq x\leq t_2$, we throw terms of $\sigma_{k}(x-t_i)$ where $i\geq 2$:
    \begin{align*}
        \mathrm{spline}_{k,t}(x)=\left(\prod_{l=2}^{k+1}t_l\right) B_{j,k,t}(x)
        =t_2\cdots t_{k+1}\frac{t_{k+2}-t_1}{\prod_{l=2}^{k+2}(t_l-t_1)}\sigma_{k}(x-t_1)
        =\sigma_{k}(x).
    \end{align*}
\end{proof}

\begin{theorem*}[Theorem \ref{th-p2k}, PowerMLP is a subset of KAN over interval]
\label{apd-p2k}
Use notations in Theorem \ref{th-k4p}.
    For any $E\in\mathbb{R}_+$, it holds
    \begin{equation*}
    \mathcal{P}_{d,w,k,p}\subset 
    {\mathcal{K}}_{2d,2w,k,2,\mathcal{O}(kp)}
    \hbox{ {\rm over} } [-E,E]^n.
    \end{equation*}
\end{theorem*}
\begin{proof}
    It suffices to prove that a PowerMLP layer
    \begin{equation}
        \label{apd:powermlp}
        \mathbf{z}=\alpha b(\mathbf{x})+\sigma_{k}(\omega\mathbf{x}+\gamma)
    \end{equation}
    can be represented by two KAN layers over bounded intervals. Suppose the input of the network $\mathbf{x}_0\in[-E,E]^{n_0}\subset\mathbb{R}^{n_0}$, then the input of each PowerMLP layer is bounded since each parameter of PowerMLP is bounded and each operation in PowerMLP will not cause infinity for bounded input.
    Thus, we denote the input of the layer as $\mathbf{x}\in[-E',E']^{n}\subset\mathbb{R}^{n}$ and output as $\mathbf{z}\in\mathbb{R}^m$. From Lemma \ref{lem:affine},  spline function can represent any transformation over bounded intervals with any $G$.
    Here we set $G=2,   t_0=-E', t_2=E'$. Then, for the first layer of KAN, we set coefficients before the basis function to be $0$ and choose the activation function as follows:
    \begin{equation*}
    \phi_{1,q,p}(x_p)=
    \left\{\begin{aligned}
        &\omega_{q,p}x_p+\gamma_{q,p},&\mathrm{for\,}1\leq q\leq m,\\
        &\delta_{q-m,p}x_p,&\mathrm{for\,}1\leq q-m\leq n,\\
    \end{aligned}\right.
    \end{equation*}
    where $\delta_{ij}$ equals to $1$ if $i=j$ and $0$ otherwise. Thus, we have $y_q=\sum_{p=1}^n\omega_{q,p}x_p+\gamma_{q,p}$ for $1\leq q\leq m$ and $y_q=x_{q-m}$ for $m+1\leq q\leq m+n$. It can be easily proved that $\mathbf{y}$ is also bounded, say in $[-E'',E'']$.

    In the second layer, we represent the addition of ReLU-$k$ activation and the basis function. By Lemma \ref{lem:repu}, spline function can represent  $\sigma_k$ over bounded intervals.
    Here we set $t_0=-E'', t_2=E''$.
    So for $1\leq r\leq m$, we set $\phi_{2,r,q}(y_q)$ as follows:
    \begin{align*}
        \phi_{2,r,q}(y_q)=
    \left\{\begin{aligned}
            &\delta_{r,q}\sigma_k(y_q),&\mathrm{for\,}1\leq q\leq m,\\
            &\alpha_{r,q-m}b(y_q),&\mathrm{for\,}1\leq q-m\leq n.\\
        \end{aligned}\right.
    \end{align*}
    By direct computation, we get the $r$-th component of the output $\mathbf{z}$ to be
    \begin{equation*}
        z_r=\sigma_k(y_r)+\sum_{q=m+1}^{m+n}\alpha_{r,q-m}b(y_q)
        =\sigma_k\left(\sum_{p=1}^n\omega_{r,p}x_p+\gamma_{r,p}\right)+\sum_{p=1}^{n}\alpha_{r,p}b(x_p),
    \end{equation*}
    which is the same as the output of Eq. \eqref{apd:powermlp}.

    In this way, the first layer has $mn+n$ activation functions and each activation function has $k+3$ parameters, totaling $(mn+n)(k+3)$ nonzero parameters. While the second layer has $mn$ activation functions that have only $1$ parameter,  and $m$ activation functions that have $k+3$  parameters respectively, totaling $mn+m(k+3)$ nonzero parameters.
    In summary, the two-layer KAN has $m+n$ width and $(k+4)mn+(k+3)(m+n)$ nonzero parameters. Since both KAN and PowerMLP are composition of layers, we prove that a PowerMLP of $w$ width, $d$ depth, $p=O(mn)$ nonzero parameters can be represented as a KAN with same order, $2d$ depth, $2w$ width and $\mathcal{O}(kp)$ nonzero parameters.
\end{proof}

\subsection{Calculation Process of Section 4.4}
Note that FLOPs of ReLU is $0$. One layer of MLP can be written as $\mathrm{ReLU}(\omega\mathbf{x}+\gamma)$, where
$\mathbf{y}=\omega\mathbf{x}$ contains $d_{in}d_{out}$ times multiplication and $(d_{in}-1)d_{out}$ times addition, and $\mathbf{y}+\gamma$ contains $d_{out}$ times addition, totaling $2d_{in}d_{out}$ FLOPs.

FLOPs of one layer of KAN can be found in \cite{yu2024kanmlpfairercomparison}.

One layer of PowerMLP can be written as
\begin{equation*}
    \sigma_k(\omega\mathbf{x}+\gamma)+\alpha b(\mathbf{x}).
\end{equation*}
FLOPs of $\mathbf{y}=\omega\mathbf{x}+\gamma$ is $2d_{in}d_{out}$ that is calculated before. FLOPs of $\mathbf{r}= \sigma_k(\mathbf{u})$ is $(k-1)d_{out}$. FLOPs of $\mathbf{s}=\alpha b(\mathbf{x})$ includes $d_{in}\lambda$ for the basis function and $(d_{in}-1)d_{out}$ for the multiplication of $\alpha$. The final addition $\mathbf{r}+\mathbf{s}$ needs $d_{out}$ FLOPs. So the total FLOPs of one layer of PowerMLP is
$    4 d_{in} d_{out}+(k-1) d_{out}+\lambda d_{in}.$

\section{Experiments}
\subsection{Details of Experiments}
All KANs used in experiments are the latest version (0.2.5) up to 2024-8-14. The main code of PowerMLP is included in the supplementary materials:
\begin{itemize}
    \item We construct PowerMLP network in \texttt{powermlp.py},
    \item Dataset including special functions, Titanic, Income, Spam, AG\_NEWS, CoLA, MNIST, SVHN, CIFAR-10 input in \texttt{data\_input.py},
    \item We include an extra \texttt{grid\_search.py} for easily searching on network shape and learning rate,
    \item Code for each experiment are jupyter notebooks in folder \texttt{experiments}.
\end{itemize}
In all experiments, we use the default configuration in KAN's paper to conduct experiments on KAN. For MLP and PowerMLP, we use Adam optimizer and a modified warm-up learning rate scheduler for training and use grid search on 10 different learning rates from $1\times10^{-4}$ to $1\times10^{-1}$ for better results.

\subsubsection{AI for Science Tasks}
In function fitting, we choose two sizes for the networks:
\textbf{(1) Small size: } KAN with $k=3,G=3$, shape $[2,1,1]$, and $24$ parameters; MLP with shape $[2,6,1]$ and $25$ parameters; $3$-order PowerMLP with shape $[2,4,1]$ and $25$ parameters;
\textbf{(2) Large size:} KAN with $k=3,G=100$, shape $[2,2,1,1]$ and $735$ parameters; MLP with shape $[2,32,18,1]$ and $709$ parameters; $3$-order PowerMLP with shape $[2,32,8,1]$ and $689$ parameters. We train each network for $5000$ epochs on random seed $42$, $114$, $514$. In knot theory, we take the KAN result from \citet{liu2024kan}, which used a KAN with shape $[17,1,14]$, $(3,3)$-spline function and $248$ parameters. Correspondingly, we use $3$-order PowerMLP with shape $[17,4,14]$ and $210$ parameters. Each network is trained for $50$ epochs.

\subsubsection{More Complex Tasks}
Shape and number of parameters of KAN, MLP and PowerMLP are shown in the following Table~\ref{tab:my_label}. KAN use $(3,3)$-spline and PowerMLP is $3$-order. Each network is trained for $500$ epochs.
\begin{table}[h]
    \centering
    \begin{tabular}{c|c c c c}
        \hline
        \multirow{2}{*}{Dataset} & \multicolumn{3}{c}{shape} & \multirow{2}{*}{\#Params}\\
         & KAN & MLP & PowerMLP & \\
        \hline
        Titanic & $[9,1,2]$ & $[9,8,2]$ & $[9,4,2]$ & \textasciitilde $100$ \\
        Income & $[108,1,2]$ & $[108,8,2]$ & $[108,4,2]$ & \textasciitilde $900$ \\
        Spam & $[100,1,2]$ & $[100,8,2]$ & $[100,4,2]$ & \textasciitilde $800$ \\
        AG\_NEWS & $[1000,32,4]$ & $[1000,256,4]$ & $[1000,128,4]$ & \textasciitilde $2.6\times10^{5}$ \\
        MNIST & $[784,8,8,10]$ & $[784,64,32,10]$ & $[784,32,32,10]$ & \textasciitilde $5\times10^{4}$ \\
        SVHN & $[1024,16,16,10]$ & $[1024,128,64,10]$ & $[1024,64,64,10]$ & $1.4\times10^{5}$ \\
        \hline
    \end{tabular}
    \caption{Shape and number of parameters of KAN, MLP and PowerMLP }
    \label{tab:my_label}
\end{table}

\subsubsection{Training Time} Experiments are conducted on a single NVIDIA GeForce RTX 4090 GPU with repeating each task $10$ times to take an average, and networks in each task are trained
with the same learning rate and hyperparameters. To avoid other influencing factors, we do not use \texttt{torch.utils.data.DataLoader} as in other experiments. The whole training data are in a large tensor and we input the whole training data tensor at one time in each epoch, which means a full-batch training. We start timing before the forward inference of the first epoch and stop timing after the parameter update of the last epoch.

\subsubsection{Ablation Study}
All the network are trained with Adam optimizer, searching on $10$ different learning rates for the best result. For better comparison, we do not include the input and output layer here; thus, depth and width here denote the number and the max width of hidden layers, which is different from the main paper but does not influence understanding. For the fixed-width experiment, we set the hidden layer width of all networks to be $4$. For the fixed-depth experiment, we search the width from $4$ to $32$ and convert variation of width to variation of number of parameters.

\subsection{Supplementary Experiments}
To supplement, we conduct two additional experiments, in which we replace MLP to PowerMLP in a CNN to solve the computer vision task and a RNN to solve the text classification task. The code of this section is provided in \texttt{modified\_net.py}.
\begin{figure}[!h]
    \centering
    \includegraphics[width=0.9\linewidth]{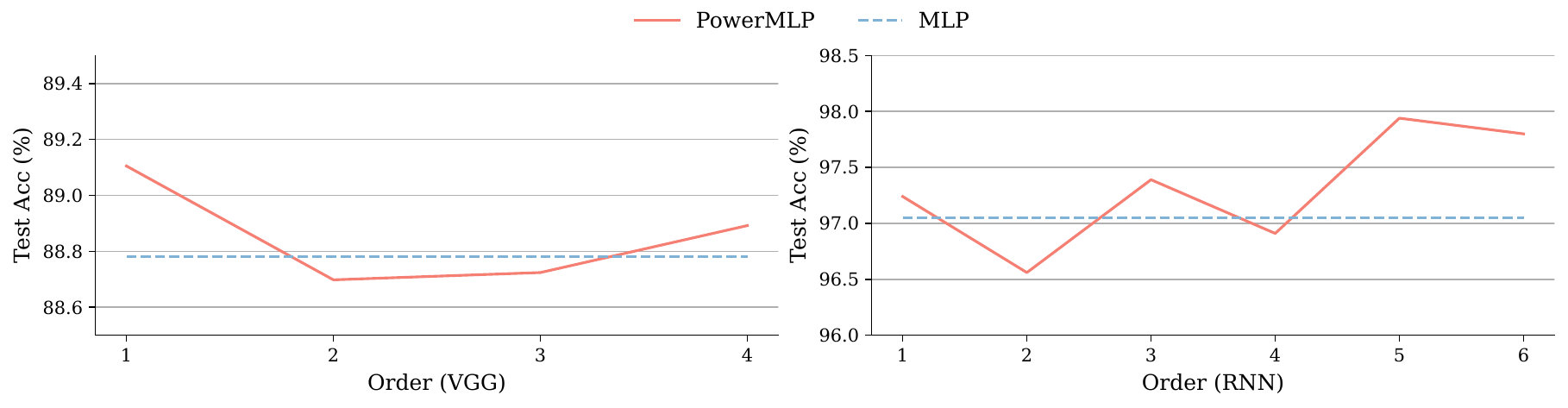}
    \caption{Supplementary Experiments.}
    \label{fig:supp}
\end{figure}

\subsubsection{Computer Vision}
We replace the MLP to PowerMLP in a small VGG network with 4 convolutional layers for CIFAR-10 classification.
We train normal VGG and PowerMLP-adapted VGG with Adam optimizer for $500$ epochs, searching on $10$ different learning rate for the best result. Additionally, we change the order of PowerMLP from $1$ to $5$, to find the influence of different order, while $5$-order fails to converge in all random seeds we have tried.

In the right of Figure~\ref{fig:supp}, blue line shows the test accuray of normal VGG and red solid line shows the variation of test accuray of PowerMLP-adapted VGG relative to the order of PowerMLP. With order increasing from $2$ to $4$, test accuracy raises accordingly, indicating the increasing expressive ability. Interestingly, order of $1$ provides an unexpected improvement in test accuracy, which may benefit from a better convergence of PowerMLP. Overall, since convolutional layer play the main role in processing images, the difference in testing accuracy is not significant.

\subsubsection{Text Classification}
We replace the MLP to PowerMLP in a small RNN with a 2-layer LSTM for CoLA classification.
We train normal RNN and PowerMLP-adapted RNN with Adam optimizer for $100$ epochs, searching on $10$ different learning rate for the best result. Additionally, we change the order of PowerMLP from $1$ to $6$, to find the influence of different order.

In the left of Figure~\ref{fig:supp}, blue line shows the test accuray of normal RNN and red solid line shows the variation of test accuray of PowerMLP-adapted RNN relative to the order of PowerMLP.
Overall, PowerMLP can improve the performance of RNNs, with higher order resulting in better performance. However, the improvement is not significant and will reach a limit as order increasing, since the LSTM part of the network also influences the performance a lot.

\end{document}